\RequirePackage{fix-cm}
\documentclass[natbib,smallextended]{svjour3}       
\smartqed  

\usepackage{natbib,amssymb,amsmath,stmaryrd,bussproofs,graphicx}

\newenvironment{scprooftree}[1]%
  {\gdef\scalefactor{#1}\begin{center}\proofSkipAmount \leavevmode}%
  {\scalebox{\scalefactor}{\DisplayProof}\proofSkipAmount \end{center} }

\newcommand{\type}[1]{#1~\mathit{type}}
\newcommand{\pair}[2]{\langle #1 , #2 \rangle}
\newcommand{\fst}[1]{\mathsf{fst}(#1)}
\newcommand{\snd}[1]{\mathsf{snd}(#1)}
\newcommand{\lam}[2]{\lambda #1.\,#2}
\newcommand{\set}{\mathsf{Set}}
\newcommand{\Ent}{\mathsf{E}}
\newcommand{\sem}[1]{\llbracket \text{#1} \rrbracket}
\newcommand{\require}[3]{\mathsf{require} ~ #1 : #2 ~ \mathsf{in} ~ #3}
\newcommand{\letexpr}[4]{\mathsf{let} ~ #1:#2 = #3 ~ \mathsf{in} ~ #4}
\newcommand{\ctx}[1]{#1~\mathit{ctx}}
\newcommand{\sig}[1]{#1~\mathit{sig}}
\newcommand{\ELAB}[1]{\textsc{elab}\!\left(#1\right)}
\newcommand{\RGm}{\mathrm{\Gamma}}
\newcommand{\RSg}{\mathrm{\Sigma}}
\newenvironment{dprooftree}{}{\DisplayProof}
\newcommand{\deriv}[1]{\AxiomC{$\mathcal{#1}$}\noLine}
\newcommand{\Eval}[2]{#1\Rightarrow #2}
\newcommand{\Member}[2]{#1\in #2}
\newcommand{\MemberEq}[3]{#1 = #2 \in #3}

\newcommand{\IsTrue}[1]{#1~\mathit{true}}

\allowdisplaybreaks

\begin{document}

\title{Dependent Types for Pragmatics}


\author{Darryl McAdams \and Jonathan Sterling}


\institute{Darryl McAdams \at
              \email{darryl@languagengine.co}\\
          \and
          Jonathan Sterling \at
            \email{jon@jonmsterling.com}
}

\date{Received: date / Accepted: date}

\maketitle

\begin{abstract}
In this paper, we present an extension to Martin-L\"of's Intuitionistic Type
Theory which gives natural solutions to problems in pragmatics, such as
pronominal reference and presupposition. Our approach also gives a simple
account of donkey anaphora without resorting to exotic scope extension of the
sort used in Discourse Representation Theory and Dynamic Semantics, thanks to
the proof-relevant nature of type theory.

\end{abstract}

\section*{Introduction}
\label{intro}

To begin with, we give a brief overview of the meaning explanations for
Intuitionistic Type Theory in Section~\ref{sec:dependent-types}, and introduce
the standard connectives. Section~\ref{sec:dependent-types-for-pragmatics}
establishes the intended meanings of pronouns and determiners under the
dependent typing discipline, and introduces an extension to the type theory
(namely our $\mathit{require}$ rule) which assigns them these meanings in the
general case.  We first give a computational justification of
$\mathit{require}$ in light of the meaning explanation, and then give a
proof-theoretic justification by showing how to eliminate $\mathsf{require}$
expressions from terms by induction on the demonstrations of their
well-typedness.

Finally, Section~\ref{sec:discussion} wraps up with a discussion of further
extensions that could be made to the framework, on both theoretical and
empirical grounds.

\section{Type Theory and its Meaning Explanation}
\label{sec:dependent-types}

Intuitionistic Type Theory is an approach to first-order and higher-order
logic, based on a computational justification called the \emph{verificationist
meaning explanation}. First, an untyped and open-ended programming language
(also called a computation system) is established with a big-step operational
semantics, given by the judgment $\Eval{M}{M'}$. Then, a type is defined by
specifying how to form a canonical member (``verification''), and when two such
canonical members are considered equal. Finally, membership $\Member{M}{A}$ is
evident when $\Eval{M}{M'}$ such that $M'$ is a canonical member of $A$.

In this setting, then, the introduction rules follow directly from the
definitions of the types, and the elimination rules are explained by showing
how one may transform the evidence for their premises into the evidence for
their conclusions. For a more detailed exposition of the verificationist
meaning explanation for intuitionistic first order logic, see
\cite{siena.lectures}; the meaning explanation for full dependent type theory
is given in \cite{Martin-Lof-1979} and \cite{Martin-Lof:1984}.

\subsection{The Connectives of Type Theory}

The two main connectives of type theory are the dependent pair $(x:A)\times B$
and the dependent function $(x:A)\to B$, where $x$ may occur free in $B$.
\footnote{In this paper, we opt to use the notation $(x : A) \times B$ and $(x : A)
\to B$ in place of the more common $\mathrm{\Sigma} x : A. B$ and
$\mathrm{\Pi} x : A. B$, respectively, in order to emphasize that these are
merely dependent versions of pairs and functions. This notation was first
invented in the Nuprl System \citep{Constable:1986:IMN:10510}.}

\subsubsection{Dependent pairs}
To define the dependent pair type, we first introduce several new terms into
the computation system, together with their canonical forms:
\begin{gather*}
  \begin{dprooftree}
    \AxiomC{}
    \UnaryInfC{$\Eval{(x:A)\times B}{(x:A)\times B}$}
  \end{dprooftree}
  \qquad
  \begin{dprooftree}
    \AxiomC{}
    \UnaryInfC{$\Eval{\pair{M}{N}}{\pair{M}{N}}$}
  \end{dprooftree}\\[6pt]
  \begin{dprooftree}
    \AxiomC{$\Eval{P}{\pair{M}{N}}$}
    \AxiomC{$\Eval{M}{M'}$}
    \BinaryInfC{$\Eval{\fst{P}}{M'}$}
  \end{dprooftree}
  \qquad
  \begin{dprooftree}
    \AxiomC{$\Eval{P}{\pair{M}{N}}$}
    \AxiomC{$\Eval{N}{N'}$}
    \BinaryInfC{$\Eval{\snd{P}}{N'}$}
  \end{dprooftree}
\end{gather*}

Then, we define the type $(x:A)\times B$ (presupposing $\type{A}$ and
$x:A\vdash \type{B}$) by declaring $\pair{M}{N}$ to be a canonical member under
the circumstances that $\Member{M}{A}$ and $\Member{N}{[M/x]B}$, where $[M/x]B$
stands for the substitution of $N$ for $x$ in $B$; moreover, $\pair{M}{N}$ and
$\pair{M'}{N'}$ are equal canonical members in case $\MemberEq{M}{M'}{A}$ and
$\MemberEq{N}{N'}{[M/x]B}$.

The formation and introduction rules for dependent pairs are immediately evident by this
definition:
\begin{gather*}
  \begin{dprooftree}
    \AxiomC{$\RGm\vdash\type{A}$}
    \AxiomC{$\RGm,x:A\vdash\type{B}$}
    \RightLabel{$\times$F}
    \BinaryInfC{$\RGm\vdash\type{(x:A)\times B}$}
  \end{dprooftree}
  \qquad
  \begin{dprooftree}
    \AxiomC{$\RGm\vdash\Member{M}{A}$}
    \AxiomC{$\RGm\vdash\Member{N}{[M/x]B}$}
    \RightLabel{$\times$I}
    \BinaryInfC{$\RGm\vdash\Member{\pair{M}{N}}{(x:A)\times B}$}
  \end{dprooftree}
\end{gather*}

The elimination rules for the dependent pair are as follows:
\begin{gather*}
  \begin{dprooftree}
    \AxiomC{$\RGm\vdash\Member{P}{(x:A)\times B}$}
    \RightLabel{$\times$E$_1$}
    \UnaryInfC{$\RGm\vdash\Member{\fst{P}}{A}$}
  \end{dprooftree}
  \qquad
  \begin{dprooftree}
    \AxiomC{$\RGm\vdash\Member{P}{(x:A)\times B}$}
    \RightLabel{$\times$E$_2$}
    \UnaryInfC{$\RGm\vdash\Member{\snd{P}}{[\fst{P}/x]B}$}
  \end{dprooftree}
\end{gather*}

\begin{proof}
  It suffices to validate the elimination rules in case $\RGm\equiv\cdot$; then, by
  hypothesis and inversion of the meaning of membership, we have
  $\Eval{P}{\pair{M}{N}}$ such that $\Member{M}{A}$ and $\Member{N}{[M/x]B}$. By
  the reduction rule for $\fst{\pair{M}{N}}$ and the meaning of membership,
  $\times\text{E}_1$ is immediately evident; because reduction is confluent, we
  know that $[M/x]B$ is computationally equal to $[\fst{P}/x]B$, whence
  $\times\text{E}_2$ becomes evident.\qed
\end{proof}

\subsubsection{Dependent Functions}
The dependent function type $(x:A)\to B$ is defined analogously. First, we
augment the computation system with new operators:
\begin{gather*}
  \begin{dprooftree}
    \AxiomC{}
    \UnaryInfC{$\Eval{(x:A)\to B}{(x:A)\to B}$}
  \end{dprooftree}
  \qquad
  \begin{dprooftree}
    \AxiomC{}
    \UnaryInfC{$\Eval{\lam{x}{M}}{\lam{x}{M}}$}
  \end{dprooftree}\\[6pt]
  \begin{dprooftree}
    \AxiomC{$\Eval{F}{\lam{x}{M}}$}
    \AxiomC{$\Eval{[N/x]M}{M'}$}
    \BinaryInfC{$\Eval{F\,N}{M'}$}
  \end{dprooftree}
\end{gather*}

Next, we define the type $(x:A)\to B$ (presuppose $\type{A}$ and
$x:A\vdash\type{B}$ by declaring that $\lam{x}{M}$ shall be a canonical member
under the circumstances that $x:A\vdash\Member{M}{B}$, and moreover, that
$\lam{x}{M}$ and $\lam{x}{N}$ shall be equal as canonical members under the
circumstances that $x:A\vdash\MemberEq{M}{N}{B}$.

Just as before, the formation and introduction rules for the dependent function
type are immediately evident:
\begin{gather*}
  \begin{dprooftree}
    \AxiomC{$\RGm\vdash\type{A}$}
    \AxiomC{$\RGm,x:A\vdash\type{B}$}
    \RightLabel{$\to$F}
    \BinaryInfC{$\RGm\vdash\type{(x:A)\to B}$}
  \end{dprooftree}
  \qquad
  \begin{dprooftree}
    \AxiomC{$\RGm,x:A\vdash\Member{M}{B}$}
    \RightLabel{$\to$I}
    \UnaryInfC{$\RGm\vdash\Member{\lam{x}{M}}{(x:A)\to B}$}
  \end{dprooftree}
\end{gather*}

The elimination rule is intended to be the following:
\begin{prooftree}
  \AxiomC{$\RGm\vdash\Member{F}{(x:A)\to B}$}
  \AxiomC{$\RGm\vdash\Member{M}{A}$}
  \RightLabel{$\to$E}
  \BinaryInfC{$\RGm\vdash\Member{F\,M}{[M/x]B}$}
\end{prooftree}

\begin{proof}
  It suffices to consider the case where $\RGm\equiv\cdot$. By hypothesis, we
  have that $\Eval{F}{\lam{x}{E}}$ such that $x:A\vdash\Member{E}{B}$; then,
  the reduction rule is applicable, yielding $\Eval{F\,M}{N}$. By the meaning
  of hypothetico-general judgment, we may deduce $\Member{N}{[M/x]B}$.\qed
\end{proof}

\subsection{Justifying the $\mathit{let}$ Rule}

Most programming languages have something called a $\mathsf{let}$ expression, which
satisfies a rule like the following:
\begin{prooftree}
  \AxiomC{$\RGm\vdash\Member{M}{A}$}
  \AxiomC{$\RGm,x:A\vdash\Member{N}{B}$}
  \AxiomC{$x\notin FV(B)$}
  \RightLabel{$\mathit{let}$}
  \TrinaryInfC{$\RGm\vdash\Member{\letexpr{x}{A}{M}{N}}{B}$}
\end{prooftree}

We may justify this rule by extending our operational semantics with a rule for
the non-canonical $\mathsf{let}$ operator:
\begin{prooftree}
  \AxiomC{$\Eval{[M/x]N}{N'}$}
  \UnaryInfC{$\Eval{\letexpr{x}{A}{M}{N}}{N'}$}
\end{prooftree}

Then, the $\mathit{let}$ rule is valid under the meaning explanation.
\begin{proof}
  It suffices to consider the case that $\RGm\equiv\cdot$. By the meaning of
  membership under hypothetico-general judgment, we have $\Eval{[M/x]N}{N'}$
  such that $N'$ is a canonical member of the type $[M/x]B$. \qed
\end{proof}

\subsection{Alternative Meaning Explanations}

The standard meaning explanation for type theory is called
\emph{verificationist} because the types are defined by stating how to form a
canonical member (i.e.\ a canonical verification); in this setting, the
introduction rules are evident by definition, and the elimination rules must be
shown to be \emph{locally sound} with respect to the introduction rules.  This
is what we have done above.

An alternative approach is to define a type by its \emph{uses}, and have the
elimination rules be evident by definition; then, the introduction rules must
be shown to be \emph{locally complete} with respect to the elimination rules.
This is called the \emph{pragmatist} meaning explanation.

Finally, following Dummett's notion of \emph{logical harmony}, one may choose
to explain the connectives by appealing to both their introduction and
elimination rules, requiring that they cohere mutually through local soundness
and local completeness \citep{Pfenning:2002}.

\section{Dependent Types for Pragmatics}
\label{sec:dependent-types-for-pragmatics}

In Dynamic Semantics, the discourse ``A man
walked in. He sat down.'' would be represented by a proposition like the following:
\[
  (\exists x:\Ent. \mathit{Man}\,x \land \mathit{WalkedIn}\,x) \land
\mathit{SatDown}\,x
\]

In standard presentations of semantics, of course, the above would be a
malformed proposition, because $x$ is out of scope in the right conjunct,
however in Dynamic Semantics, the scope of existentials is extended
aritificially to make this a well-formed proposition. Following Sundholm's 1986
revelation, however, in a dependently typed setting we may assign such a
sentence the following meaning:
\[
  (p : (x:\Ent) \times \mathit{Man}\,x \times \mathit{WalkedIn}\,x) \times \mathit{SatDown}\,(\fst{p})
\]

Rather than modifying the behavior of existentials, which under the dependent
typing discipline become pairs, we instead use a dependent pair type in place
of the conjunction. Conjunctions would become pair types regardless, but by
using an explicitly dependent pair, we license the right
conjunct to refer to not only the propositional content of the left conjunct,
but also to the \emph{witnesses} of the existentially quantified proposition,
by way of projection.

The semantics for \textit{a}, \textit{man}, \textit{walked in}, and \textit{sat
down} are, in simplified form, just direct translations from the usual semantic
representations:
\begin{align*}
  \sem{a} &\in (\Ent \to \set) \to (\Ent \to \set) \to \set\\
  \sem{a} &= \lam{P}{\lam{Q}{(x:\Ent) \times P\,x \times Q\,x}}\\[6pt]
  \sem{man} &\in \Ent \to \set\\
  \sem{man} &= \mathit{Man}\\[6pt]
  \sem{walked in} &\in \Ent \to \set\\
  \sem{walked in} &= \mathit{WalkedIn}\\[6pt]
  \sem{sat down} &\in \Ent \to \set\\
  \sem{sat down} &= \mathit{SatDown}
\end{align*}

Conjunction (in the form of sentence sequencing) is easily assigned a meaning
in a similar way:
\begin{align*}
  \sem{S$_1$. S$_2$.} &\in \set\\
  \sem{S$_1$. S$_2$.} &= (p : \sem{S$_1$}) \times \sem{S$_2$}
\end{align*}

But when we come to the meaning of the pronoun \textit{he}, we run into a
problem. What could it possibly be? For the example that we are currently
considering, we need $\sem{he} = \fst{p}$, but this is not in general a
solution for arbitrary occurrences of the pronoun, since it depends on the name
and type of the free variable $p$.

Consider now the discourse ``A man walked in. The man (then) sat down.'' The
use of \textit{the man} in the right conjunct, instead of \textit{he},
introduces presuppositional content via the definite determiner. Ideally, the
semantics of this should be nearly identical to those of the previous example
(modulo $\beta$ reduction). By giving \textit{the} a dependently typed meaning,
we can achieve this relatively simply:
\begin{align*}
  \sem{the} &\in (P : \Ent \to \set) \to (x:\Ent) \to P\,x \to E\\
  \sem{the} &= \lam{P}{\lam{x}{\lam{q}{x}}}
\end{align*}

The first argument to \textit{the} is simply the predicate, which in this case
will be $\mathit{Man}$. The second argument is an entity, and the third is an
inhabitant of the type $P\,x$, i.e. a witness that $P\,x$ holds. Therefore we
would want:
\[
  \sem{the man} = (\lam{P}{\lam{x}{\lam{q}{x}}})~(\mathit{Man}\,(\fst{p}))~(\fst{\snd{p}}) =_\beta \fst{p}
\]

The term $\fst{p}:\Ent$ is the man referred to in the left conjunct. $\snd{p}$
is a witness that he is in fact a man, and that he walked in, and so
$\fst{\snd{p}}$ is the witness that he is a man. The argument $\fst{p}$ is, in
effect, the solution to the presupposition induced by \textit{the}, and
$\fst{\snd{p}}$ is the witness that the propositional component of the
presupposition holds.

The next two pairs of examples go hand in hand. Consider the classic donkey
anaphora sentences ``If a farmer owns a donkey, he beats it.'' and ``Every
farmer who owns a donkey beats it.'' A typical Dynamic Semantics approach might
assign these sentences the following meaning:
\[
  \forall x:\Ent. \mathit{Farmer}\,x \land (\exists y:\Ent. \mathit{Donkey}\,y \land \mathit{Owns}\,x\,y) \Rightarrow \mathit{Beats}\,x\,y
\]

In the dependently typed setting, we can assign a similar meaning, but which
has a more straightforward connection to the syntax (for convenience, we define
the subscript $p_i$ to project the $i$th element of a right nested tuple):
\[
  (p : (x:\Ent) \times \mathit{Farmer}\,x \times (y:\Ent) \times \mathit{Donkey}\,y \times \mathit{Owns}\,x\,y) \to \mathit{Beats}\,p_1\,p_3
\]

The lexical entries for the content words and pronouns should be obvious at
this point, but for \textit{if}, \textit{a}, and \textit{every} we can define:
\begin{align*}
  \sem{if} &\in \set \to \set \to \set\\
  \sem{if} &= \lam{P}{\lam{Q}{(p : P) \to Q}}\\[6pt]
  \sem{a} &\in (\Ent \to \set) \to (\Ent \to \set) \to \set\\
  \sem{a} &= \lam{P}{\lam{Q}{(x:\Ent) \times P\,x \times Q\,x}}\\[6pt]
  \sem{every} &\in (\Ent \to Type) \to (\Ent \to \set) \to \set\\
  \sem{every} &= \lam{P}{\lam{Q}{(p : (x:\Ent) \times P\,x) \to Q (\fst{p})}}
\end{align*}
With these, we can get:
\begin{gather*}
  \begin{split}
    &\sem{a farmer owns a donkey}\\
    &\quad=(x:\Ent) \times \mathit{Farmer}\,x \times (y:\Ent) \times \mathit{Donkey}\,y \times \mathit{Owns}\,x\,y\\
    &\sem{if a farmer owns a donkey}\\
    &\quad=\lam{Q}{(p : (x:\Ent) \times \mathit{Farmer}\,x \times (y:\Ent) \times \mathit{Donkey}\,y \times \mathit{Owns}\,x\,y) \to Q}\\
    &\sem{farmer who owns a donkey}\\
    &\quad=\lam{x}{\mathit{Farmer}\,x \times (y:\Ent) \times \mathit{Donkey}\,y \times \mathit{Owns}\,x\,y}\\
    &\sem{every farmer who owns a donkey}\\
    &\quad=\lam{Q}{(p : (x:\Ent) \times \mathit{Farmer}\,x \times (y:\Ent) \times \mathit{Donkey}\,y \times \mathit{Owns}\,x\,y)\\
    &\qquad\to Q\,p_1}\\
    &\sem{beats it}\\
    &\quad=\lam{z}{\mathit{Beats}\,z\,p_3}\\
    &\sem{he beats it}\\
    &\quad=\mathit{Beats}\,p_1\,p_3
  \end{split}
\end{gather*}

We echo Sundholm's conclusion that the treatment of donkey-sentences licensed
in Martin-L\"of's type theory is not \emph{ad hoc}, but rather is reflective of
the general suitability of the framework:

\begin{quote}
  In this manner, then, the type-theoretic abstractions suffice to solve the
  problem of the pronominal back-reference in [the donkey-sentence]. It should
  be noted that there is nothing \emph{ad hoc} about the treatment, since all
  the notions used have been introduced for mathematical reasons in complete
  independence of the problem posed by [the donkey-sentence].
  \citep[p.~503]{Sundholm:1986}
\end{quote}

\subsection{Terms for Presuppositions}
\label{sec:terms-for-presuppositions}

Provided that we can devise a general mechanism to assign the meanings given
above to pronouns and definite determiners, our semantics will work just as
well as standard techniques like Discourse Representation Theory or Dynamic
Semantics, but in a well-scoped manner.

A number of possible solutions exist to do precisely this sort of thing in the
programming languages literature. Haskell's type class constraints
\citep{Marlow_haskell2010} and Agda's instance arguments \citep{Devriese:2011}
provide very similar functionality but for somewhat different purposes, so one
option would be to repurpose those ideas.

Haskell's type classes, however, depend on global reasoning and an anti-modular
coherence condition which makes them inapplicable to our use-case, since in
general there will be many solutions to a presupposition. Agda's instance
arguments are closer to our needs, but we believe that a simpler approach is
warranted which lends direct insight into the semantics and pragmatics of
presuppositions.

The approach we will take here involves a new operator ($\mathsf{require}$)
that binds variables for presupposed parts of an expression. Terms, contexts
and signatures are defined as follows:

\[
  \begin{array}{lrclll}
    \textit{Terms}            & M, N, A, B & ~::=~ & x &|~ \set_i \\
                  &            & |  & (x : A) \to B &|~ \lam{x}{M} &|~ M N \\
                  &            & |  & (x : A) \times B &|~ \pair{M}{N} &|~ \fst{M} ~|~ \snd{M} \\
                  &            & |  & \multicolumn{3}{l}{\require{x}{A}{M}} \\
    \textit{Contexts}         & \RGm     & ~::=~ & \cdot ~|~ \RGm, x : A\\
    \textit{Signatures}       & \RSg     & ~::=~ & \cdot ~|~ \RSg, x : A
  \end{array}
\]

The new term $\require{x}{A}{M}$ should be understood to mean roughly ``find
some $x : A$ and make it available in $M$.'' In this version of
type theory, we replace the judgment $\type{A}$ with membership in a universe,
$\Member{A}{\set_i}$; except where ambiguous, we omit the level from a universe
expression, writing $\set$.

Lexical constants (e.g.~$Man$, $Own$, etc.) are to be contained in a
\emph{signature} $\RSg$, whereas the context $\RGm$ is reserved for local
hypotheses. The use of signatures to carry the constants of a theory originates
from the Edinburgh Logical Framework, where individual logics were represented
as signatures of constants which encode their syntax, judgments and rule
schemes \citep{Harper:1993, HARPER:2007}. Then the basic forms of judgment are
as follows:
\[
  \begin{array}{rlll}
    &\vdash&\sig{\RSg} &\qquad\text{$\RSg$ is a valid signature}\\
    &\vdash_\RSg&\ctx{\RGm} &\qquad\text{$\RGm$ is a valid context}\\
    \RGm&\vdash_\RSg &M:A &\qquad\text{$M$ has type $A$}
  \end{array}
\]

In context validity judgments $\vdash_\RSg\ctx{\RGm}$, we presuppose
$\vdash\sig{\RSg}$; likewise, in typing judgments $\RGm\vdash_\RSg\Member{M}{A}$, we
presuppose $\vdash_\RSg\ctx{\RGm}$. The rules for the signature and context
validity judgments are as expected:

\begin{gather*}
  \begin{dprooftree}
    \AxiomC{}
    \UnaryInfC{$\vdash\sig{\cdot}$}
  \end{dprooftree}
  \qquad
  \begin{dprooftree}
    \AxiomC{$\vdash\ctx{\RSg}$}
    \AxiomC{$\cdot\vdash_\RSg\Member{A}{\set}$}
    \AxiomC{$x\notin \RSg$}
    \TrinaryInfC{$\vdash\sig{\RSg, x : A}$}
  \end{dprooftree}
\end{gather*}

\begin{gather*}
  \begin{dprooftree}
    \AxiomC{}
    \UnaryInfC{$\vdash_\RSg\ctx{\cdot}$}
  \end{dprooftree}
  \qquad
  \begin{dprooftree}
    \AxiomC{$\vdash_\RSg\ctx{\RGm}$}
    \AxiomC{$\RGm \vdash_\RSg \Member{A}{\set}$}
    \AxiomC{$x\notin \RGm\cup\RSg$}
    \TrinaryInfC{$\vdash_\RSg\ctx{\RGm, x : A}$}
  \end{dprooftree}
\end{gather*}

Constants and hypotheses may be projected from signatures and contexts respectively:
\begin{gather*}
  \begin{dprooftree}
    \AxiomC{}
    \RightLabel{const}
    \UnaryInfC{$\RGm\vdash_{\RSg, x : A, \RSg'} \Member{x}{A}$}
  \end{dprooftree}
  \qquad
  \begin{dprooftree}
    \AxiomC{}
    \RightLabel{hyp}
    \UnaryInfC{$\RGm, x : A, \RGm' \vdash_\RSg \Member{x}{A}$}
  \end{dprooftree}
\end{gather*}

The inference rules for the familiar terms are the usual ones:
\begin{prooftree}
  \AxiomC{$i < j$}
  \RightLabel{cumulativity}
  \UnaryInfC{$\RGm\vdash_\RSg \Member{\set_i}{\set_j}$}
\end{prooftree}
\begin{prooftree}
  \AxiomC{$\RGm \vdash_\RSg \Member{A}{\set}$}
  \AxiomC{$\RGm, x : A \vdash_\RSg\Member{B}{\set}$}
  \RightLabel{$\to$F}
  \BinaryInfC{$\RGm \vdash_\RSg\Member{(x : A) \to B}{\set}$}
\end{prooftree}
\begin{prooftree}
  \AxiomC{$\RGm, x : A \vdash_\RSg \Member{M}{B}$}
  \RightLabel{$\to$I}
  \UnaryInfC{$\RGm \vdash_\RSg \Member{\lam{x}{M}}{(x : A) \to B}$}
\end{prooftree}
\begin{prooftree}
  \AxiomC{$\RGm \vdash_\RSg \Member{M}{(x : A) \to B}$}
  \AxiomC{$\RGm \vdash_\RSg \Member{N}{A}$}
  \RightLabel{$\to$E}
  \BinaryInfC{$\RGm \vdash_\RSg \Member{M\,N}{[N/x]B}$}
\end{prooftree}
\begin{prooftree}
  \AxiomC{$\RGm \vdash_\RSg \Member{A}{\set}$}
  \AxiomC{$\RGm, x : A \vdash_\RSg\Member{B}{\set}$}
  \RightLabel{$\times$F}
  \BinaryInfC{$\RGm \vdash_\RSg\Member{(x : A) \times B}{\set}$}
\end{prooftree}
\begin{prooftree}
  \AxiomC{$\RGm \vdash_\RSg\Member{M}{A}$}
  \AxiomC{$\RGm \vdash_\RSg\Member{N}{[M/x]B}$}
  \RightLabel{$\times$I}
  \BinaryInfC{$\RGm \vdash_\RSg\Member{\pair{M}{N}}{(x : A) \times B}$}
\end{prooftree}
\begin{prooftree}
  \AxiomC{$\RGm \vdash_\RSg\Member{P}{(x : A) \times B}$}
  \RightLabel{$\times$E$_1$}
  \UnaryInfC{$\RGm \vdash_\RSg\Member{\fst{P}}{A}$}
\end{prooftree}
\begin{prooftree}
  \AxiomC{$\RGm \vdash_\RSg\Member{P}{(x : A) \times B}$}
  \RightLabel{$\times$E$_2$}
  \UnaryInfC{$\RGm \vdash_\RSg\Member{\snd{P}}{[\fst{P}/x]B}$}
\end{prooftree}

The only inference rule which is new deals with presuppositions:
\begin{prooftree}
  \AxiomC{$\RGm \vdash_\RSg\Member{M}{A}$}
  \AxiomC{$\RGm \vdash_\RSg\Member{[M/x]N}{B}$}
  \AxiomC{$x\notin FV(B)$}
  \RightLabel{$require$}
  \TrinaryInfC{$\RGm \vdash_\RSg\Member{\require{x}{A}{N}}{B}$}
\end{prooftree}

We can now provide a semantics for pronouns and definite determiners:
\begin{align*}
  \sem{he}  &= \require{x}{\Ent}{x}\\
  \sem{it}  &= \require{x}{\Ent}{x}\\
  \sem{the} &= \lam{P}{\require{x}{\Ent}{(\require{p}{P\,x}{x}})}
\end{align*}

Now let us reconsider our examples with the new semantics:
\begin{gather*}
  \begin{split}
    &\sem{A man walked in. He sat down.}\\
    &\quad=(p : (x:\Ent) \times \mathit{Man}\,x \times \mathit{WalkedIn}\,x) \times \mathit{SatDown}(\require{y}{\Ent}{y})\\
    &\sem{A man walked in. The man (then) sat down.}\\
    &\quad=(p : (x:\Ent) \times \mathit{Man}\,x \times \mathit{WalkedIn}\,x)\\
    &\qquad\times \mathit{SatDown}\,(\require{y}{\Ent}{(\require{q}{\mathit{Man}\,y}{y})})\\
    &\sem{If a farmer owns a donkey, he beats it.}\\
    &\quad=(p : (x:\Ent) \times \mathit{Farmer}\,x \times (y:\Ent) \times \mathit{Donkey}\,y \times \mathit{Owns}\,x\,y)\\
    &\qquad\to \mathit{Beats}~(\require{z}{\Ent}{z})~(\require{w}{\Ent}{w})\\
    &\sem{Every farmer who owns a donkey beats it.}\\
    &\quad=(p : (x:\Ent) \times \mathit{Farmer}\,x \times (y:\Ent) \times \mathit{Donkey}\,y \times \mathit{Owns}\,x\,y)\\
    &\qquad\to \mathit{Beats}~p_1~(\require{w}{\Ent}{w})
  \end{split}
\end{gather*}

\subsection{Computational Justification of the $\mathit{require}$ Rule}
A $\mathsf{require}$ expression is, in essence, the same as a $\mathsf{let}$
expression, as found in many programming languages, except that the definiens
is supplied by fiat. Its formation rule is a bit strange, of course, because
the presupposition's witness appears in the premises but not in the conclusion;
from a type-theoretic perspective, however, this is acceptable.

For instance, many of the rules of Computational Type Theory
\citep{Allen2006428, Constable:1986:IMN:10510} strategically forget their
premises, yielding novel and useful constructions such as \emph{set types}
$\{x:A\mid B(x)\}$ and \emph{squash types} $\downarrow A$. On the other hand,
this causes the typing judgment to become synthetic \citep{analytic-synthetic}:
the evidence for the judgment is not recoverable from the statement of the
judgment itself, but must be constructed by the knowing subject.

The introduction of types whose members do not contain their own typing
derivations is completely justified under the verificationist meaning
explanation, but this does not suffice to explain the $\mathit{require}$ rule,
which is not part of the definition of a new connective. Intuitionistic
validity for $\mathit{require}$ must be established in the same way as the
validity of $\mathit{let}$, i.e.\ by computation. However, it is clear that we
cannot devise an effective operation which produces out of thin air a solution to
an arbitrary presupposition if there is one, since this would entail deciding the truth
of any proposition (and solving Turing's Halting Problem).

This, however, does not pose an obstacle for an intuitionistic justification of
this rule, since assertion acts are tensed \citep{vanAtten2007}.  Because
evaluation itself is an assertion, we may explain the meaning of the judgment
$\Eval{\require{x}{A}{N}}{N'}$ by appealing to the state of knowledge at the
time of assertion.

Informally, at time $n$, the value of $\require{x}{A}{N}$ shall be, for any
witness of $\Member{M}{A}$ that has been experienced by time $n$, the value of
the substitution $[M/x]N$. It should be noted, then, that the computational behavior
of this operator is non-deterministic, since in general the truth of $A$
shall have been experienced in many different ways (corresponding to the number
of known solutions to the presupposition).

This explanation suffices to validate the $\mathit{require}$ rule in light of
the meaning explanation which was propounded in
Section~\ref{sec:dependent-types}:

\begin{prooftree}
  \AxiomC{$\RGm \vdash_\RSg\Member{M}{A}$}
  \AxiomC{$\RGm \vdash_\RSg\Member{[M/x]N}{B}$}
  \AxiomC{$x\notin FV(B)$}
  \RightLabel{$require$}
  \TrinaryInfC{$\RGm \vdash_\RSg\Member{\require{x}{A}{N}}{B}$}
\end{prooftree}

\begin{proof}
  It suffices to validate the rule in case $\RGm\equiv\cdot$; then, we must
  show that $\Eval{\require{x}{A}{N}}{N'}$ such that $\Member{N'}{B}$. By our
  definition, the $\mathsf{require}$ term shall have a value in case a witness
  for $A$ has been experienced; but this is already the case from the hypothesis
  $\Member{M}{A}$. By inverting the hypothesis $\Member{[M/x]N}{B}$, we
  have $\Eval{[M/x]N}{N'}$ such that $N'$ is a canonical member of $B$. \qed
\end{proof}

This concludes the intuitionistic justification of the $\mathit{require}$ rule.

\subsubsection*{Discussion and Related Work}

The augmentation of our computation system with a non-deterministic oracle
($\mathsf{require}$) may be viewed as a computational effect. The behavior of
$\mathsf{require}$ is defined separately at every type $A$, and therefore
cannot be computed by a recursive algorithm; this ``infinitely large''
definition is acceptable in type theory because we make no a priori commitment
to satisfy Church's Thesis, which states that every effective operation is
recursive. Accepting the possibility of effective but non-recursive operations
leads to a property called \emph{computational open-endedness}
\citep{Howe:1991}, and endows the intuitionistic continuum with the full
richness of the classical one \citep{vanAtten2007}.

The explanation of the computational behavior of the $\mathsf{require}$
operator is related to the Brouwer's theory of the Creating Subject, and may be
seen as a ``proof-relevant'' version of Kripke's Schema. Sundholm explains how
the Kreisel-Myhill axiomatization of the Creating Subject may be treated
propositionally in Martin-L\"of's type theory, relative to the existence of a
uniform verification object for Kripke's Schema \citep{Sundholm:2014}.

In the same way as we have exploited the intensional character of assertion
acts in intuitionistic mathematics, \citet{Coquand:2012} prove the uniform
continuity principle by adding a generic element $\mathsf{f}$ to their
computation system, representing a Cohen real; their interpretation results in
a non-trivial combination of realizability with Beth/Kripke semantics.

Finally, \citet{Rahli:2015} add two computational effects to type theory
(dynamic symbol generation and exception handling), and use them to prove
Brouwer's continuity theorem and justify bar induction on monotone bars.

\subsection{Elaboration}
\label{sec:elaboration}

In addition to the computational justification of $\mathit{require}$, we may
give a proof-theoretic justification by showing how to eliminate all instances
of $\mathsf{require}$ from a term via elaboration.\footnote{From a formalistic
perspective, the elaboration is all that is needed to justify the rule.} To
this end, we will define a meta-operation $\ELAB{\mathcal{D}}$ which transforms
a derivation $\mathcal{D}::\RGm \vdash\Member{M}{A}$ into an elaborated term
$M'$ which is like $M$ but with $\mathsf{require}$ expressions replaced by
their solutions. We define the operation inductively over the structure of the
derivations as follows:

\begin{align*}
  \ELAB{
    \begin{dprooftree}
      \AxiomC{}
      \RightLabel{const}
      \UnaryInfC{$\RGm \vdash_\RSg\Member{x}{A}$}
    \end{dprooftree}
  } &\leadsto x\\
  \ELAB{
    \begin{dprooftree}
      \AxiomC{}
      \RightLabel{hyp}
      \UnaryInfC{$\RGm\vdash_\RSg\Member{x}{A}$}
    \end{dprooftree}
  } &\leadsto x\\
  \ELAB{
    \begin{dprooftree}
      \AxiomC{}
      \RightLabel{cumulativity}
      \UnaryInfC{$\RGm \vdash_\RSg\Member{\set_i}{\set_j}$}
    \end{dprooftree}
  } &\leadsto \set_i\\
  \ELAB{
    \begin{dprooftree}
      \deriv{D}
      \UnaryInfC{$\RGm \vdash_\RSg\Member{A}{\set}$}
      \deriv{E}
      \UnaryInfC{$\RGm, x : A \vdash_\RSg\Member{B}{\set}$}
      \RightLabel{$\to$F}
      \BinaryInfC{$\RGm \vdash_\RSg\Member{(x : A) \to B}{\set}$}
    \end{dprooftree}
  } &\leadsto (x : \ELAB{\mathcal{D}}) \to \ELAB{\mathcal{E}}\\
  \ELAB{
    \begin{dprooftree}
      \deriv{D}
      \UnaryInfC{$\RGm, x : A \vdash_\RSg\Member{M}{B}$}
      \RightLabel{$\to$I}
      \UnaryInfC{$\RGm \vdash_\RSg \Member{\lam{x}{B}}{(x : A) \to B}$}
    \end{dprooftree}
  } &\leadsto \lam{x}{\ELAB{\mathcal{D}}}\\
  \ELAB{
    \begin{dprooftree}
      \deriv{D}
      \UnaryInfC{$\RGm \vdash_\RSg \Member{M}{(x : A) \to B}$}
      \deriv{E}
      \UnaryInfC{$\RGm \vdash_\RSg \Member{N}{A}$}
      \RightLabel{$\to$E}
      \BinaryInfC{$\RGm \vdash_\RSg \Member{M\,N}{[N/x]B}$}
    \end{dprooftree}
  } &\leadsto \ELAB{\mathcal{D}}\,\ELAB{\mathcal{E}}\\
  \ELAB{
    \begin{dprooftree}
      \deriv{D}
      \UnaryInfC{$\RGm \vdash_\RSg \Member{A}{\set}$}
      \deriv{E}
      \UnaryInfC{$\RGm, x : A \vdash_\RSg \Member{B}{\set}$}
      \RightLabel{$\times$F}
      \BinaryInfC{$\RGm \vdash_\RSg \Member{(x : A) \times B}{\set}$}
    \end{dprooftree}
    } &\leadsto (x : \ELAB{\mathcal{D}}) \times \ELAB{\mathcal{E}}\\
  \ELAB{
    \begin{dprooftree}
      \deriv{D}
      \UnaryInfC{$\RGm \vdash_\RSg \Member{M}{A}$}
      \deriv{E}
      \UnaryInfC{$\RGm \vdash_\RSg \Member{N}{[M/x]B}$}
      \RightLabel{$\times$I}
      \BinaryInfC{$\RGm \vdash_\RSg \Member{\pair{M}{N}}{(x : A) \times B}$}
    \end{dprooftree}
  } &\leadsto \pair{\ELAB{\mathcal{D}}}{\ELAB{\mathcal{E}}}\\
  \ELAB{
    \begin{dprooftree}
      \deriv{D}
      \UnaryInfC{$\RGm \vdash_\RSg \Member{P}{(x : A) \times B}$}
      \RightLabel{$\times$E$_1$}
      \UnaryInfC{$\RGm \vdash_\RSg \Member{\fst{P}}{A}$}
    \end{dprooftree}
  } &\leadsto \fst{\ELAB{\mathcal{D}}}\\
  \ELAB{
    \begin{dprooftree}
      \deriv{D}
      \UnaryInfC{$\RGm \vdash_\RSg \Member{P}{(x : A) \times B}$}
      \RightLabel{$\times$E$_2$}
      \UnaryInfC{$\RGm \vdash_\RSg \Member{\snd{P}}{[\fst{P}/x]B}$}
    \end{dprooftree}
  } &\leadsto \snd{\ELAB{\mathcal{D}}}\\
  \ELAB{
    \begin{dprooftree}
      \deriv{D}
      \UnaryInfC{$\RGm \vdash_\RSg \Member{M}{A}$}
      \deriv{E}
      \UnaryInfC{$\RGm \vdash_\RSg \Member{[M/x]N}{B}$}
      \RightLabel{require}
      \BinaryInfC{$\RGm \vdash_\RSg \Member{\require{x}{A}{N}}{B}$}
    \end{dprooftree}
  } &\leadsto \ELAB{\mathcal{E}}
\end{align*}

The most crucial rule is the last one --- the preceding ones simply define
elaboration by induction on the structure of derivations other than those for
$\mathsf{require}$ expressions. For a $\mathsf{require}$ expression, however,
we substitute the proof of the presupposed content for the variable in the body
of the $\mathsf{require}$ expression.

It is evident that the elaboration process preserves type.

\begin{theorem}
  Given a derivation $\mathcal{D}::\RGm\vdash_\RSg \Member{M}{A}$, there exists another
  derivation $\mathcal{D'}::\RGm\vdash_\RSg \ELAB{\mathcal{D}}:A$.
\end{theorem}
\begin{proof}
  By induction on the structure of $\mathcal{D}$. \qed
\end{proof}

An example of elaboration in action is necessary, so consider again the
sentence ``A man walked in. He sat down.'' Prior to elaboration, its meaning
will be:
\[
  (p : (x:\Ent) \times \mathit{Man}\,x \times \mathit{WalkedIn}\,x) \times \mathit{SatDown}(\require{x}{\Ent}{x})
\]

Now let $\RSg = \mathit{Man}:\Ent\to\set, \mathit{WalkedIn}:\Ent\to\set,
\mathit{SatDown}:\Ent\to\set$. After constructing a derivation that the above
type is a $\set$ under the signature $\RSg$, we can elaborate the associated
term. The left conjunct elaborates to itself, so we will not look at that, but
the elaboration for the right conjunct is more interesting. The derivation for
the right conjunct, letting \(\RGm = p : (x:\Ent) \times \mathit{Man}\,x \times
\mathit{WalkedIn}\,x\), is:
\begin{prooftree}
  \AxiomC{}
  \RightLabel{const}
  \UnaryInfC{$\RGm \vdash_\RSg \Member{\mathit{SatDown}}{\Ent \to \set}$}
  \AxiomC{$\cdots$}
  \deriv{D}
  \UnaryInfC{$\RGm \vdash_\RSg \Member{\fst{p}}{\Ent}$}
  \RightLabel{require}
  \BinaryInfC{$\RGm \vdash_\RSg \Member{\require{x}{\Ent}{x}}{\Ent}$}
  \RightLabel{$\to$E}
  \BinaryInfC{$\RGm \vdash_\RSg \Member{\mathit{SatDown}(\require{x}{\Ent}{x})}{\set}$}
\end{prooftree}

Inductively, we get:
\begin{align*}
\ELAB{
  \begin{dprooftree}
    \AxiomC{}
    \RightLabel{const}
    \UnaryInfC{$\RGm \vdash_\RSg \Member{\mathit{SatDown}}{\Ent \to \set}$}
  \end{dprooftree}
} &\leadsto\mathit{SatDown}\\
\ELAB{
  \begin{dprooftree}
    \deriv{D}
    \UnaryInfC{$\RGm \vdash_\RSg \Member{\fst{p}}{\Ent}$}
  \end{dprooftree}
} &\leadsto \fst{p}
\end{align*}
For the $\mathsf{require}$ expression's elaboration, we substitute $\fst{p}$ in for $x$ in $x$ to get the following:

\begin{align*}
\ELAB{
  \begin{dprooftree}
    \AxiomC{$\cdots$}
    \deriv{D}
    \UnaryInfC{$\RGm \vdash_\RSg \Member{\fst{p}}{\Ent}$}
    \RightLabel{require}
    \BinaryInfC{$\RGm \vdash_\RSg \Member{\require{x}{\Ent}{x}}{\Ent}$}
  \end{dprooftree}
} &\leadsto \fst{p}
\end{align*}

And finally the elaboration of whole subderivation yields
$\mathit{SatDown}(\fst{p})$, and so the complete derivation yields
\[
  (p : (x:\Ent) \times \mathit{Man}\,x \times \mathit{WalkedIn}\,x) \times \mathit{SatDown}(\fst{p})
\]
which is the meaning we had wanted.

A similar proof for ``A man walked in. The man (then) sat down.'' can be given,
with an extra non-trivial branch for $\mathit{Man}(\fst{p})$. Focusing just on
the subproof for \textit{the man}, we have the following typing derivation:
%
\begin{scprooftree}{0.75}
  \deriv{E}
  \UnaryInfC{$\RGm \vdash_\RSg \Member{\fst{p}}{\Ent}$}
  \AxiomC{}
  \RightLabel{hyp}
  \UnaryInfC{$\RGm \vdash_\RSg \Member{p}{(x:\Ent) \times \mathit{Man}\,x \times \mathit{WalkedIn}\,x}$}
  \RightLabel{$\times$E$_2$}
  \UnaryInfC{$\RGm \vdash_\RSg \Member{\snd{p}}{\mathit{Man}(\fst{p}) \times \mathit{WalkedIn}(\fst{p})}$}
  \RightLabel{$\times$E$_1$}
  \UnaryInfC{$\RGm \vdash_\RSg \Member{\fst{\snd{p}}}{\mathit{Man}(\fst{p})}$}
  \deriv{E}
  \UnaryInfC{$\RGm \vdash_\RSg \Member{\fst{p}}{\Ent}$}
  \RightLabel{require}
  \BinaryInfC{$\RGm \vdash_\RSg \Member{\require{q}{\mathit{Man}(\fst{p})}{\fst{p}}}{\Ent}$}
  \RightLabel{require}
  \BinaryInfC{$\RGm \vdash_\RSg \Member{\require{x}{\Ent}{(\require{q}{\mathit{Man}\,x}{x})})}{\Ent}$}
\end{scprooftree}
This similarly elaborates to $\fst{p}$ just as the subproof for \textit{he} did
before.

Elaboration for ``If a farmer owns a donkey, he beats it.'' and ``Every farmer
who owns a donkey beats it.'' unfolds in a similar fashion, with the
elaboration of the antecedent \((x:\Ent) \times \mathit{Farmer}\,x \times
(y:\Ent) \times \mathit{Donkey}\,y \times \mathit{Owns}\,x\,y\) being trivial.
The consequent \(\mathit{Beats}~(\require{z}{\Ent}{z})~(\require{w}{\Ent}{w})\)
breaks down into three subproofs, one for the predicate $\mathit{Beats}$ which
elaborates trivially, and the two $\mathsf{require}$ subproofs which elaborate
like the previous pronominal examples. The only difference now is that the
context licenses more options for the proofs.

Keen eyes will notice, however, that there should be four solutions, because
both $\mathsf{require}$ expressions demand something of type $\Ent$ --- the
words \textit{he} and \textit{it} have no gender distinction in the semantics.
This is left as an unspecified part of the framework, as there are a number of
options for resolving gender constraints. Two options that are immediately
obvious are 1) make $\Ent$ itself a primitive function \(\Ent : \mathit{Gender}
\to \set\) and then specify a gender appropriately, or 2) add another
$\mathsf{require}$ expression so that, for example, \(\sem{he} =
\require{x}{\Ent}{(\require{p}{\mathit{Masc}\,x}{x})}\) and provide appropriate
axioms (possibly simply by deferring to other cognitive systems for judging
gender). The former solution is akin to how certain versions of HPSG treat
gender as a property of indices not of syntactic elements.

\section{Discussion and Related Work}
\label{sec:discussion}

In the previous sections, we have described an approach to pronominal and
presuppositional pragmatics based on dependent types, as an alternative to DRT
and Dynamic Semantics. The main difference from a standard dependently typed
$\lambda$ calculus is the addition of $\mathsf{require}$ expressions, and an
elaboration process to eliminate them.

\subsection{Contextual Modal Type Theory}

Another approach would be to eliminate $\mathsf{require}$ expressions by
adopting Contextual Modal Type Theory (CMTT) to support metavariables for
presuppositions \citep{Nanevski:2008}. From our perspective, our system relates
to CMTT in the same way that programming directly with computational effects
relates to programming in a monad. Indeed, the intuitionistic justification of
the $\mathit{require}$ rule obtains by adding an intensional effect to our
computation system, whereas a CMTT-based solution would involve solving
presuppositions after the fact by providing a substitution.

A modal extension can also provide an interesting solution to another
well-known problem in pragmatics. Consider the sentence ``John will pull the
rabbit out of the hat'' when said of a scene that has 3 rabbits, 3 hats, but
only a single rabbit in a hat. This sentence seems to be pragmatically
acceptable and unambiguous, despite there being neither a unique rabbit nor a
unique hat. In the framework given above, there should be 9 possible ways of
resolving the presuppositions, leading to pragmatic ambiguity. A simple
modality (approximately a possibility modality), however, can make sense of
this: if the assertion of such a sentence presupposes that the sentence can be
true via a modality (i.e. to assert $P$ is to presuppose $\diamond P$), then
there is only one way to solve the rabbit and hat presuppositions which would
also make it possible to resolve the possibility presupposition --- pick the
rabbit that is in a hat, and the hat that the rabbit is in --- yielding a
unique, unambiguous meaning. Whether this belongs in the semantics-pragmatics
or in some higher system (such as a Gricean pragmatics) is debatable, but that
such a simple solution is readily forthcoming at all speaks to the power of the
above framework.

\subsection{Ranta's Type-Theoretical Grammar}

The most representative use of dependent types in linguistics is Aarne Ranta's
work on type-theoretical grammar \citep{Ranta:1994}, where pronominal meaning
is given via inference rules for each particular pronoun or other
presuppositional form. For example, the pronoun ``he'' can be explained by giving
the following rules:

\begin{gather*}
  \begin{dprooftree}
    \AxiomC{$a : man$}
    \UnaryInfC{$he(a) : man$}
  \end{dprooftree}
  \qquad
  \begin{dprooftree}
    \AxiomC{$a : man$}
    \UnaryInfC{$he(a) = a : man$}
  \end{dprooftree}
\end{gather*}

The first is a typing rule, and the latter is the associated equality rule
which reflects computation. This approach can generalize to any
sort of presuppositional content, but leaves the question of the meaning of
such expressions somewhat unanswered, since these interpretations presuppose
that we have already understood the solution to the presupposition.

A discourse context without any possible antecedent will not merely cause a
\emph{type membership} error, as in the system presented in this paper, but
will instead not have a meaning at all, as no term can be produced. We consider
this an undesirable property in a semantic formalism. Interlocutors will
typically not fail to understand sentences with unknown antecedents. For
example, when presented with just the sentence ``he's tall'' out of the blue,
most people will respond by asking ``who's tall?'', rather than by failing to
find a meaning at all. To capture this, it's necessary for the sentence to have
a meaning --- that is, a term produced by the parser --- even in the absence of
that meaning computing to a value which the listener shall judge to be a
canonical proposition.

In practice, in order to give meanings to anaphora which do not presuppose
knowledge of their antecedents, such a theory must be extended with selection
operators, such as Bekki's $@$-operator \citep{bekki:2014} or our
$\mathsf{require}$ operator.  This technique, of separating the assignment of
meanings from the assertion that they are propositional, is based directly upon
Martin-L\"of's reconstruction of propositional well-formedness as a judgment,
rather than a mere matter of grammar \citep{siena.lectures}.

\subsection{Bekki's @-operator}

To assign meanings to anaphora, \citet{bekki:2014} pursued an approach similar
to ours, in which an oracle operator $(@_i : A)$ was added with the following
formation rule:
\begin{prooftree}
  \AxiomC{$\type{A}$}
  \AxiomC{$\IsTrue{A}$}
  \BinaryInfC{$\Member{(@_i : A)}{A}$}
\end{prooftree}

The index $i$ allows an expression to share a presupposition with another,
which is a very useful extension that might be added to our framework.
Following our computational interpretation of $\mathsf{require}$, we see our
operator as essentially a call-by-value analogue to Bekki's $(@_i : A)$, since
in $\require{x}{A}{N}$, the presupposition $x:A$ must be resolved before $N$
shall be reduced.

We believe that our $\mathsf{require}$ operator is suggestive of the
interactive nature of presupposition resolution; indeed, it is possible to see
$\require{x}{A}{N}$ as a dialogue, in which one party requests an $A$ to fill
the hole in $N$ --- and so it seems likely that the oracle's choice of a
felicitous $\Member{M}{A}$ shall be based in part on the sense of the intended
construction $N(x)$, and we may recover the form $(@_0:A)$ as the special case
$\require{x}{A}{x}$.

\begin{acknowledgements}
  The second author thanks Mark Bickford, Bob Harper and Bob Constable for
  illuminating discussions on choice sequences, Church's Thesis, and
  computational open-endedness. We thank our reviewers for their constructive
  feedback and references to related work.
\end{acknowledgements}

\bibliographystyle{spbasic}
\bibliography{dtfp}

\begin{thebibliography}{20}
\providecommand{\natexlab}[1]{#1}
\providecommand{\url}[1]{{#1}}
\providecommand{\urlprefix}{URL }
\expandafter\ifx\csname urlstyle\endcsname\relax
  \providecommand{\doi}[1]{DOI~\discretionary{}{}{}#1}\else
  \providecommand{\doi}{DOI~\discretionary{}{}{}\begingroup
  \urlstyle{rm}\Url}\fi
\providecommand{\eprint}[2][]{\url{#2}}

\bibitem[{Allen et~al(2006)Allen, Bickford, Constable, Eaton, Kreitz, Lorigo,
  and Moran}]{Allen2006428}
Allen S, Bickford M, Constable R, Eaton R, Kreitz C, Lorigo L, Moran E (2006)
  Innovations in computational type theory using {Nuprl}. Journal of Applied
  Logic 4(4):428 -- 469, towards Computer Aided Mathematics

\bibitem[{Bekki(2014)}]{bekki:2014}
Bekki D (2014) Representing anaphora with dependent types. In: Asher N,
  Soloviev S (eds) Logical Aspects of Computational Linguistics, Lecture Notes
  in Computer Science, vol 8535, Springer Berlin Heidelberg, pp 14--29

\bibitem[{Constable et~al(1986)Constable, Allen, Bromley, Cleaveland, Cremer,
  Harper, Howe, Knoblock, Mendler, Panangaden, Sasaki, and
  Smith}]{Constable:1986:IMN:10510}
Constable RL, Allen SF, Bromley HM, Cleaveland WR, Cremer JF, Harper RW, Howe
  DJ, Knoblock TB, Mendler NP, Panangaden P, Sasaki JT, Smith SF (1986)
  Implementing Mathematics with the Nuprl Proof Development System.
  Prentice-Hall, Inc., Upper Saddle River, NJ, USA

\bibitem[{Coquand and Jaber(2012)}]{Coquand:2012}
Coquand T, Jaber G (2012) A computational interpretation of forcing in type
  theory. In: Dybjer P, Lindström S, Palmgren E, Sundholm G (eds) Epistemology
  versus Ontology, Logic, Epistemology, and the Unity of Science, vol~27,
  Springer Netherlands, pp 203--213

\bibitem[{Devriese and Piessens(2011)}]{Devriese:2011}
Devriese D, Piessens F (2011) On the bright side of type classes: Instance
  arguments in agda. In: Proceedings of the 16th ACM SIGPLAN International
  Conference on Functional Programming, ACM, New York, NY, USA, ICFP '11, pp
  143--155

\bibitem[{Harper and Licata(2007)}]{HARPER:2007}
Harper R, Licata D (2007) Mechanizing metatheory in a logical framework.
  Journal of Functional Programming 17(4-5):613--673

\bibitem[{Harper et~al(1993)Harper, Honsell, and Plotkin}]{Harper:1993}
Harper R, Honsell F, Plotkin G (1993) A framework for defining logics. Journal
  of the ACM 40(1):143--184

\bibitem[{Howe(1991)}]{Howe:1991}
Howe D (1991) On computational open-endedness in {Martin-L\"{o}f's} type
  theory. In: Logic in Computer Science, 1991. LICS '91., Proceedings of Sixth
  Annual IEEE Symposium on, pp 162--172

\bibitem[{Marlow(2010)}]{Marlow_haskell2010}
Marlow S (2010) {Haskell 2010 Language Report}

\bibitem[{Martin-L{\"o}f(1982)}]{Martin-Lof-1979}
Martin-L{\"o}f P (1982) Constructive mathematics and computer programming. In:
  Cohen LJ, {\L o\'s} J, Pfeiffer H, Podewski KP (eds) Logic, Methodology and
  Philosophy of Science VI, Proceedings of the Sixth International Congress of
  Logic, Methodology and Philosophy of Science, Hannover 1979, Studies in Logic
  and the Foundations of Mathematics, vol 104, North-Holland, pp 153--175

\bibitem[{Martin-L\"{o}f({1984})}]{Martin-Lof:1984}
Martin-L\"{o}f P ({1984}) {Intuitionistic Type Theory}. {Bibliopolis}

\bibitem[{Martin-L{\"o}f(1994)}]{analytic-synthetic}
Martin-L{\"o}f P (1994) Analytic and synthetic judgements in type theory. In:
  Parrini P (ed) Kant and Contemporary Epistemology, The University of Western
  Ontario Series in Philosophy of Science, vol~54, Springer Netherlands, pp
  87--99

\bibitem[{Martin-L{\"o}f(1996)}]{siena.lectures}
Martin-L{\"o}f P (1996) On the meanings of the logical constants and the
  justifications of the logical laws. Nordic Journal of Philosophical Logic
  1(1):11--60

\bibitem[{Nanevski et~al(2008)Nanevski, Pfenning, and Pientka}]{Nanevski:2008}
Nanevski A, Pfenning F, Pientka B (2008) Contextual modal type theory. ACM
  Trans Comput Logic 9(3):23:1--23:49

\bibitem[{Pfenning(2002)}]{Pfenning:2002}
Pfenning F (2002) Logical frameworks---a brief introduction. In: Schwichtenberg
  H, Steinbr\"{u}ggen R (eds) Proof and System-Reliability, NATO Science
  Series, vol~62, Springer Netherlands, pp 137--166

\bibitem[{Rahli and Bickford(2015)}]{Rahli:2015}
Rahli V, Bickford M (2015) A nominal exploration of intuitionism, unpublished

\bibitem[{Ranta(1994)}]{Ranta:1994}
Ranta A (1994) Type-theoretical Grammar. Oxford University Press, Oxford, UK

\bibitem[{Sundholm(1986)}]{Sundholm:1986}
Sundholm G (1986) Proof theory and meaning. In: Gabbay D, Guenthner F (eds)
  Handbook of Philosophical Logic, Synthese Library, vol 166, Springer
  Netherlands, pp 471--506

\bibitem[{Sundholm(2014)}]{Sundholm:2014}
Sundholm G (2014) Constructive recursive functions, {Church's Thesis}, and
  {Brouwer}'s theory of the creating subject: Afterthoughts on a parisian joint
  session. In: Dubucs J, Bourdeau M (eds) Constructivity and Computability in
  Historical and Philosophical Perspective, Logic, Epistemology, and the Unity
  of Science, vol~34, Springer Netherlands, pp 1--35

\bibitem[{{van Atten}(2007)}]{vanAtten2007}
{van Atten} M (2007) {B}rouwer {M}eets {H}usserl: {O}n the {P}henomenology of
  {C}hoice {S}equences. {S}pringer

\end{thebibliography}

\end{document}